\documentclass{article}

\usepackage{microtype}
\usepackage{graphicx}
\usepackage{tikz}
\usepackage{subfigure}
\usepackage{booktabs} 

\usepackage{hyperref}

\usepackage[accepted]{icml2023}

\usepackage{amsmath}
\usepackage{amssymb}
\usepackage{mathtools}
\usepackage{amsthm}

\usepackage[capitalize,noabbrev]{cleveref}

\usepackage{amsmath,amsfonts,bm}

\def\figref#1{Figure~\ref{#1}}
\def\Figref#1{Figure~\ref{#1}}

\def\secref#1{section~\ref{#1}}
\def\Secref#1{Section~\ref{#1}}

\def\Algref#1{Algorithm~\ref{#1}}

\def\1{\bm{1}}

\def\rt{{\textnormal{t}}}

\def\rvn{{\mathbf{n}}}

\def\rvx{{\mathbf{x}}}
\def\rvy{{\mathbf{y}}}
\def\rvz{{\mathbf{z}}}
\def\rvzeta{{\mathbf{\zeta}}}

\def\vn{{\bm{n}}}

\def\vr{{\bm{r}}}

\def\vw{{\bm{w}}}
\def\vx{{\bm{x}}}
\def\vy{{\bm{y}}}
\def\vz{{\bm{z}}}

\DeclareMathAlphabet{\mathsfit}{\encodingdefault}{\sfdefault}{m}{sl}
\SetMathAlphabet{\mathsfit}{bold}{\encodingdefault}{\sfdefault}{bx}{n}

\newcommand{\E}{\mathbb{E}}

\newcommand{\R}{\mathbb{R}}

\newcommand{\N}{\mathbb{N}}
\renewcommand{\P}{\mathbb{P}}
\newcommand{\C}{\mathcal{C}}
\renewcommand{\L}{\mathrm{L}}

\newcommand{\X}{\mathcal{X}}
\newcommand{\Z}{\mathcal{Z}}
\newcommand{\T}{\mathcal{T}}
\newcommand{\id}{\mathrm{Id}}
\newcommand{\norm}[1]{\left\Vert#1\right\Vert}
\newcommand{\scal}[2]{\left\langle#1,#2\right\rangle}
\renewcommand{\vec}[1]{\mathbf{#1}}
\newcommand{\diameter}{\mathrm{diam}}
\newcommand{\prox}{\mathrm{prox}}
\newcommand{\proj}{\mathrm{proj}}
\newcommand{\tmin}{t_\mathrm{min}}
\newcommand{\tmax}{t_\mathrm{max}}
\def\hatt{{\widehat{t}}}
\newcommand{\tminh}{\hatt_\mathrm{min}}
\newcommand{\tmaxh}{\hatt_\mathrm{max}}

\newcommand{\dist}[1]{\mathcal{P}_{#1}}
\newcommand{\pdf}[1]{p_{#1}}

\theoremstyle{plain}
\newtheorem{theorem}{Theorem}[section]

\newtheorem{corollary}[theorem]{Corollary}
\theoremstyle{definition}
\newtheorem{definition}[theorem]{Definition}

\theoremstyle{remark}

\icmltitlerunning{Learning Gradually Non-convex Image Priors Using Score Matching}

\begin{document}

\twocolumn[
\icmltitle{Learning Gradually Non-convex Image Priors Using Score Matching}




\begin{icmlauthorlist}
\icmlauthor{Erich Kobler}{x}
\icmlauthor{Thomas Pock}{y}
\end{icmlauthorlist}

\icmlaffiliation{x}{Department of Neuroradiology, University Hospital Bonn, Bonn, Germany}
\icmlaffiliation{y}{Institue of Computer Graphics and Vision, Graz University of Technology, Graz, Austria}

\icmlcorrespondingauthor{Erich Kobler}{erich.kobler@ukbonn.de}

\icmlkeywords{Machine Learning, ICML}

\vskip 0.3in
]



\printAffiliationsAndNotice{}  

\begin{abstract}
In this paper, we propose a unified framework of denoising score-based models in the context of graduated non-convex energy minimization.
We show that for sufficiently large noise variance, the associated negative log density -- the energy -- becomes convex. Consequently, denoising score-based models essentially follow a graduated non-convexity heuristic.
We apply this framework to learning generalized Fields of Experts image priors that approximate the joint density of noisy images and their associated variances.
These priors can be easily incorporated into existing optimization algorithms for solving inverse problems and naturally implement a fast and robust graduated non-convexity mechanism.
\end{abstract}

\section{Introduction}


Score matching (SM,~\citealp{Hy05}) has recently seen a renewed interest in computer vision and machine learning as it allows to fit a high-dimensional parametric distribution to a given data distribution while avoiding computing the often intractable normalization constant. The basic idea is to match (in a log domain) the gradients of the parametric distribution with the gradients of the data distribution, by minimizing a least squares loss function. Interestingly, the problem of computing the gradients of the log-data distribution can be avoided by implicit SM, which results in a loss function merely depending on the gradient and Laplacian of the parametric model.

In subsequent work, \citet{Vi11} showed equivalence of these SM techniques to denoising autoencoders by introducing denoising SM, where the gradients of a slightly smoothed data distributions are computed based on Tweedie's formula~\cite{Ef11}.
\citet{SoEr19,SoEr20} introduced noise conditional score networks~(NCSNs) by conditioning a score predicting network on the noise level.

In parallel, diffusion probabilistic models~\cite{SoJa15} motivated by nonequilibrium thermodynamics evolved, which are learned to revert a diffusion process -- a Markov chain transforming a data distribution to a prior distribution by gradually adding noise.
Later, \citet{HoJa20} introduced denoising diffusion probabilistic models (DDPM) that explicitly connect diffusion probabilistic models to denoising SM in the sense that the noise level is encoded by a schedule in the diffusion time steps.
Their image-generation results were remarkable and ignited further research in the field of score-based generative models~\cite{RoBl22,DiSa22,YaZh22,HoCh22}.
Still, the underlying learning technique of all the previously introduced score-based generative models remains denoising SM.

Typically, the transition from the data distribution to the prior distribution is discretized~\citep{HoJa20,SoEr19,SoEr20}.
During inference samples are generated by traversing the reverted discrete process by means of a stochastic sampling heuristic.
In contrast, numerical stochastic differential equation (SDE) solvers can be used for sampling if the score models are learned to continuously depend on diffusion time~\citep{SoSo21}.
These solvers can further be used to solve inverse problems by considering a stochastic gradient flow of the associated posterior distribution~\citep[Eq. (14)]{SoSo21}.
Moreover, this gradient flow perspective and its approximation by a score network lead to denoising diffusion implicit models~(DDIM)~\citep{SoMe22}, which allow for a significant speed up in sampling while maintaining image quality to a large extent.

In contrast to the main body of recent works that focus on the score -- the
gradient of the log density, we focus in this work on the respective
negative log density -- the energy. We consider the problem of
learning energies of corresponding image priors on gradually noisier images, where the variance~$t$ of the noise is equal to a smoothing parameter.
Indeed adding noise to images is equivalent to sampling images from the true image distribution \emph{smoothed} by a Gaussian with variance~$t$.
Moreover, we show that under mild assumptions, there always exists a
sufficiently large smoothing parameter $\widetilde{t}$ for which the
corresponding negative log density (the energy) becomes a convex function.
Thus, denoising score-based generative models learn to approximate gradients of energies that become gradually more \emph{convex} for
increasing noise and recent inference techniques such as DDIM~\citep{SoMe22} follow the graduated non-convexity~(GNC) principle~\citep{BlZi87}, which is a widely used heuristic for avoiding local minima.

Inspired by these observations, we propose a unified energy-based
perspective of denoising score-based models through learning a
one-parameter family of energies continuously parametrized by~$t$ to approximate the smoothed negative log density of the data.
This enables easy integration of the learned prior energy into the variational approach to inverse problems by choosing sequences for~$t$, i.e., smoothing schedules.
In particular, we compare the following options:
\begin{enumerate}
  \item A joint minimization in both the image and the smoothing
    parameter.
  \item A predefined schedule such that the subsequent inference
    becomes equivalent to GNC.
  \item A task-specific schedule learned by unrolling a certain
    number of proximal gradient steps.
\end{enumerate}
Interestingly, the last option allows drawing connections to
variational networks~\cite{ChPo16,KoKl17}.

\section{A Graduated Non-convexity View of Score-based Generative Models} \label{sec:gnc}
Let $(\X,\mathfrak{F},\P)$ be a complete probability space on a compact set~$\X\subset\R^d$ with sigma algebra~$\mathfrak{F}$ and probability measure~$\P$.
We further assume an absolutely continuous probability measure with corresponding probability density function~$\pdf{\rvx}\in \C_0(\X,[0,\infty))$.
Later, we will also consider a discrete probability measure defined by the empirical distribution of a dataset~$\{\vx_i\}_{i=1}^n\subset\X$ with cardinality~$n\in\N$.
In addition, we consider smoothing parameters~$t\in\T=[\tmin,\tmax]$, where $0<\tmin<\tmax<\infty$.

In this setting, we show that score-based generative models~\citep{SoEr19,HoJa20} actually learn a graduated non-convexity (GNC) scheme~\citep{BlZi87}.
In detail, score-based generative models train a neural network to approximate the conditional score, i.e., the gradient of the log density of the data degraded by additive Gaussian noise~$s(\vy,t)\approx\nabla_\vy\log\pdf{\rvy\vert\rt}(\vy,t)$ where $\rvy=\rvx+\sqrt{\rt}\rvn$, $\rvn\sim\mathcal{N}(0,\id)$.
The key observation is that although the data becomes noisier with increasing~$t$, the corresponding probability density
\[
\pdf{\rvy\vert\rt}(\vy\vert t)=\left(\pdf{\rvx}\ast G(\vec{0},t\id)\right)(\vy)
\]
gets smoother due to the convolution with a Gaussian density~$G(\vec{0},t\id)$ and its associated energy~$F(\vy,t)=-\log\pdf{\rvy\vert\rt}(\vy,t)$ gets more \emph{convex} (in fact more quadratic), see the left plot in Figure~\ref{fig:gmmOptimization}.
Next, we prove that there indeed exists a lower bound on the variance~$\widetilde{t}$ such that the energy of the conditional density is convex for~$t\geq\widetilde{t}$.

Let $G(\bm{\mu},\Sigma)$ denote a multivariate Gaussian probability density with mean~$\bm{\mu}\in\R^d$ and symmetric positive definite covariance matrix~$\Sigma\in\R^{d\times d}$, which reads as
\[
G(\bm{\mu},\Sigma)(\vx)=\vert 2\pi\Sigma\vert^{-\frac{1}{2}}\exp\left(-\norm{\vx-\bm{\mu}}_{\Sigma^{-1}}^2\right).
\]
\begin{definition}[GMM]
Let $n\in\N$.
A Gaussian mixture model (GMM) consisting of $n$ components is defined as
\[
p = \sum_{i=1}^n w_i G(\vx_i,\Sigma_i)
\]
with means~$\{\vx_i\}_{i=1}^n\subset\R^d$, covariances~$\{\Sigma_i\}_{i=1}^n\subset\R^{d\times d}$, and weights on the unit simplex~$\vw=(w_1\ \ldots\ w_n)^\top\in\Delta^n$.
\end{definition}
Recall that for $n\to\infty$ a GMM can uniformly approximate any function in~$\C_0$~\citep{NgNg20}.
Note that an approximation w.r.t.~$\L^p$ also holds for any~$p_\rvx\in\L^p$ for~$p\in[1.\infty)$.
Consequently, our setting and most practically encountered probability density functions can be well approximated by GMMs.
Thus, we focus in the following on probability density functions induced by GMMs.

\begin{theorem}\label{thm:Fconvex}
Let~$\X\subset\R^d$ and~$\diameter(\X)<\infty$. 
Consider a GMM of the form
\[
p(\vx) = \sum_{i=1}^n w_i G(\vx_i,\Sigma_i)(\vx),
\]
where $(w_1\ \cdots \ w_n)^\top\in\Delta^n$. Assume~$\{\vx_i\}_{i=1}^n\subset\X$.
Then there exists a smoothing parameter~$\widetilde{t}\in(0,\infty)$ such that the smoothed energy
\[
F(\vx,t)\coloneqq -\log \big( (p \ast G(\vec{0}, t\id))(\vx)\big)
\]
is \emph{convex} w.r.t.~$\vx$ for all~$t\geq\widetilde{t}$.
\end{theorem}
\begin{corollary}\label{cor:empirical}
Theorem~\ref{thm:Fconvex} also holds if an empirical discrete probability measure of a dataset~$\{\vx_i\}_{i=1}^n\subset\X$, i.e.
\[
p = \frac{1}{n}\sum_{i=1}^n\delta_{\vx_i},
\]
is considered.
Here, $\delta_{\vx}$ denotes the Dirac delta measure located at~$\vx$.
\end{corollary}
The proofs can be found in~\cref{apdx:proofs}.
Note that in this paper we only consider the variance exploding setting, however, similar results hold for variance preserving schemes. 

The underlying idea of the GNC is to speed up the estimation of the global minimum of a non-convex energy.
First, this energy is approximated by a one-parameter family of energies that become more convex for increasing~$t$.
Second, an initial smoothing parameter is selected such that the energy is convex.
Third, the initial problem is efficiently solved due to increased smoothness and convexity.
Then, the parameter is reduced such that the energy becomes gradually more non-convex and is in turn minimized starting from the previous solution.
The process is repeated until~$t\to0$.

Comparing denoising score-based models to GNC, we observe the following similarities.
First, the score networks also approximate a one-parameter family of the gradient of associated energies.
Second, the schedule of the smoothing parameter is frequently a priori fixed.
Finally, as we proved above, there exists a smoothing parameter such that the associated energy becomes convex in most practical cases.

Combining both approaches results in the GNC flow described in~\Algref{alg:graduatedNC}, which is a discretization of a continuous gradient flow starting from a convex approximation.
In detail, first a decreasing sequence of smoothing parameters~$\{t_i\}_{i=0}^I$ is fixed such that~$F(\vx,t_0)$ is convex.
Then, we estimate the minimum by a single gradient step using the variance~$t_i$ as preconditioning, which is motivated by Tweedie's identity~\citep{Ro56,Ef11} and the fact that for a smooth and convex function the minimal means squared error (MMSE) estimator is close to its minimum.
Then this estimator is gradually refined by reducing the smoothing~$t_i>t_{i+1}$ and updating the MMSE estimator conditioned on the previous result.

\begin{algorithm}[t]
\caption{Graduated non-convexity flow for minimizing a smoothed family of energies~$F(\vx,t)$}\label{alg:graduatedNC}
\textbf{Step 0:} Take~$I\in\N$, choose sequence $\{t_i\}_{i=0}^I\subset[\tmin,\tmax]$ s.t.~$t_{i}>t_{i+1}$ and~$F(\vx,t_0)$ convex, $\vx_0\in\X$, $\eta>0$ \\
\textbf{Step i:} $(0\leq i< I)$ 
Approximately minimize current energy~$F(\cdot,t_i)$ using single step
\[
\vx_{i+1}=\vx_i-\eta t_i\nabla_\vx F(\vx_i,t_i)
\] 
\end{algorithm}

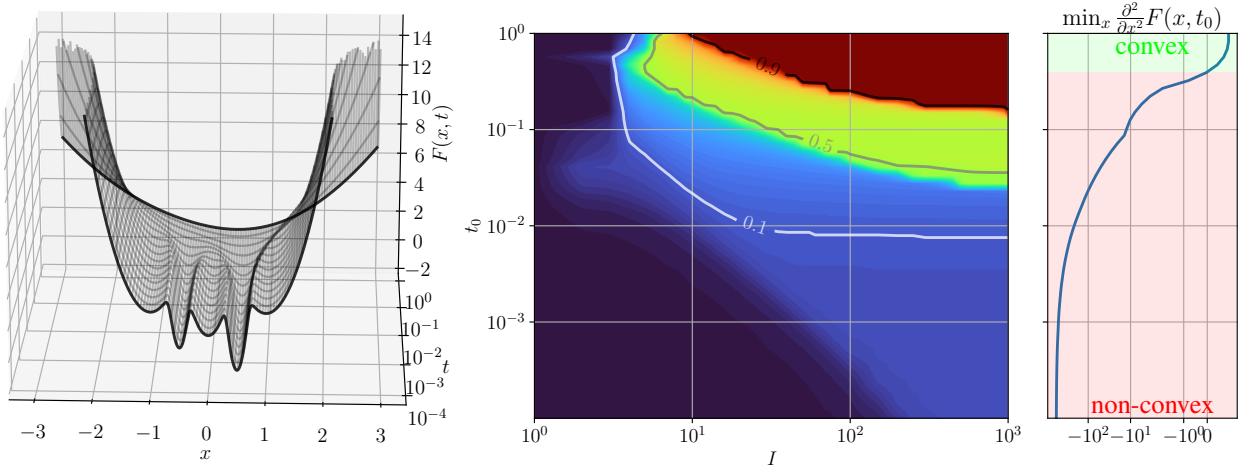
\begin{figure*}[th]
\centering
\begin{tikzpicture}[every node/.append style={inner sep=1mm}]
\node at (-8.3,0) {\includegraphics[width=.35\linewidth]{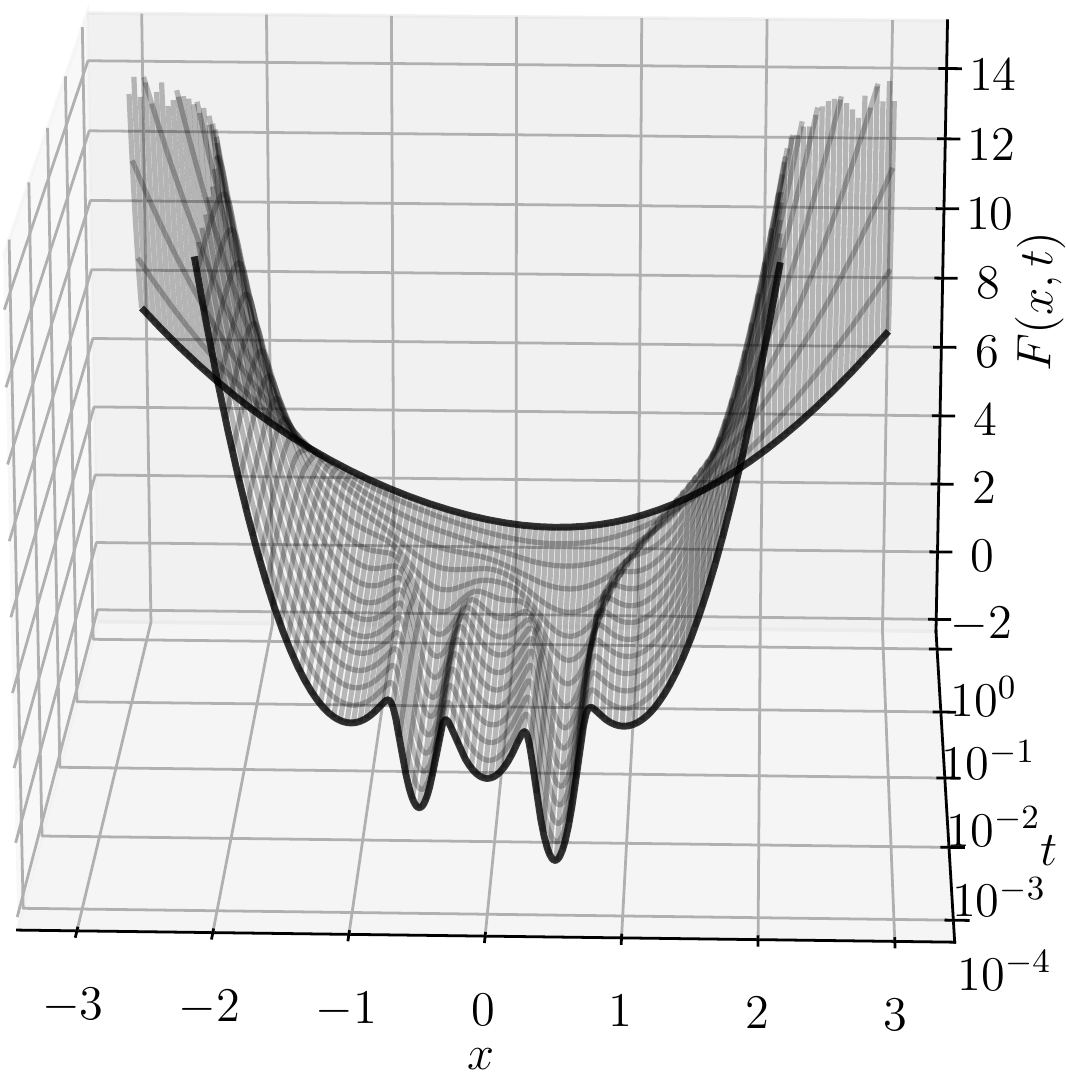}};
\node at (0,0) {\includegraphics[width=.6\linewidth]{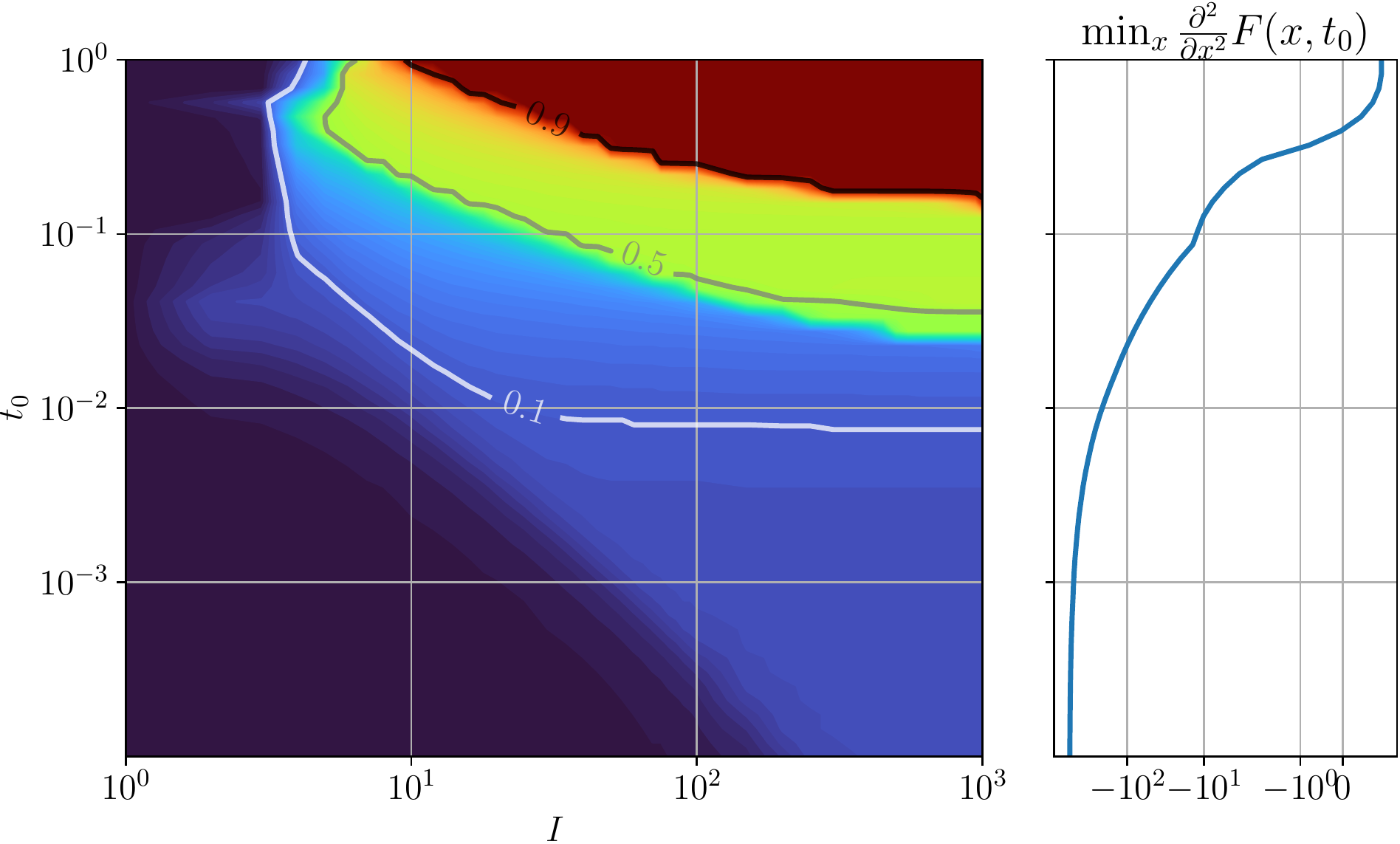}};
\fill[red, opacity=.1] (4.7-2.1,2.15) rectangle (4.7+.4,-2.45);
\node[text=red, anchor=south] at (4.,-2.5) {non-convex};
\fill[green, opacity=.1] (4.7-2.1,2.15) rectangle (4.7+.4,2.15+.5);
\node[text=green, anchor=north] at (4.,2.65) {convex};
\end{tikzpicture}
\label{fig:gmmOptimization}
\caption{Left: Illustration of the 1D example energy~\eqref{eq:gmmEx} using gray lines for different~$t\in[10^{-4},1]$.
Center: Visualization of the rate of trajectories attaining the global minimum at~$x=\frac{1}{2}$ using~$N=1\ 000$ equally spaced initial points for different~$t_0$ and~$I$.
The larger the initial smoothing~$t_0$, the fewer steps are required to obtain the perfect rate.
Right: Smallest second derivative of~$F$ w.r.t.~$x$ for different~$t_0$.
}
\end{figure*}
Next, we consider a simple 1D example to illustrate the effectiveness of the GNC flow.
In particular, we use a smoothed GMM corresponding to the energy
\begin{align} \label{eq:gmmEx}
F(x,t)=-\log\left(\sum_{i=1}^5 w_i G\left(\mu_i,\sigma_i^2+t\right)(x)\right)
\end{align}
consisting of five components over the domain $x\in[-3,3]$ and $t\in[10^{-4},1]$.
The detailed parameters are~$\bm{w}=\tfrac{1}{100}(5\ 15\ 15\ 60\ 5)$, $\bm{\mu}=(-1\ -\tfrac12\ 0\ \tfrac12\ 1)$, and $\bm{\sigma}^2=\tfrac{1}{100}(10\ 1\ 5\ 1\ 10)$.
Note that the addition of $t$ at the variances in~\eqref{eq:gmmEx} originates from the convolution of the data density with a Gaussian as considered in Theorem~\ref{thm:Fconvex}.
This energy is illustrated by the black lines at the left in~\figref{fig:gmmOptimization} for various~$t$.
As can be seen, increasing the smoothing parameter~$t$ results in smoother and more convex energies.

The plot in the center of \figref{fig:gmmOptimization} visualizes the rate of trajectories converging to the global minimum at~$x=\tfrac12$ as a function of the initial smoothing~$t_0$ and the number of steps~$I$.
For each~$t_0$ all trajectories start from~$N=1\ 000$ equally spaced initial positions~$x^0$ on $[-3,3]$ and are defined by performing $I$~GNC flow steps using a logarithmic smoothing schedule.
By comparing the different contour lines we observe that for larger initial smoothing fewer steps are required such that all trajectories converge to the global minimum.
This highlights the effectiveness of GNC.
Finally, the right plot in \figref{fig:gmmOptimization} depicts the smallest~$\frac{\partial^2}{\partial x^2}F(x,t_0)$ and thereby highlights convexity of~$F$ as a function of~$t_0$.

\section{Learning Gradually Non-convex Image Priors}
In this section, we transfer the insights gained from relations of GNC and score-based generative models to learn priors for natural images.

\subsection{FoE-like Prior Models}
Let $\vx\in\X\subset\R^d$ be an image of size~$d=m\times n\times c$, where $m$ represents its height, $n$ its width, and $c$ is the number of feature channels.
\citet{RoBl09} introduced a simple and versatile prior operating on images called fields of experts~(FoE).
This prior has been successfully applied to various inverse problems in imaging sciences and its essential building blocks are (local) convolutions that extract lower-level features as well as non-linear potential functions.
In particular, we consider the FoE model
\[
R_\mathrm{FoE}(\vx) = \scal{\vec{1}}{\left(\Phi \circ K\right)(\vx)},
\]
where the linear operator~$K\colon\R^d\to\R^{d_1}$ extracts~$N_1$ features using 2D convolution kernels~$k_i$ with~$d_1=m\times n\times N_1$.
The operator~$\Phi\colon\R^{d_1}\to\R^{d_1}$ applies to every feature channel~$k_i(\vx)$ a corresponding pixel-wise parametric non-linear potential~$\phi_i\colon\R\to\R$, and the scalar product denotes the sum over all pixels and channels.
The underlying idea is that every convolution kernel~$k_i$ specializes on a certain pattern and the associated potential~$\phi_i$ describes the corresponding energy, i.e., the negative logarithm of the density.
The non-linear functions~$\phi_i$ are typically learnable and implemented by simple parametric functions~\citep{RoBl09,ChRa14} or weighted radial basis functions~\citep{ChPo16,KoKl17}.

\subsubsection{Extending the Fields of Experts}
Due to the simplicity of the FoE prior a natural approach to include the conditioning concept of SBGMs is by affecting the potential functions to learn joint prior models.
In detail, the extended FoE reads as
\begin{align}\label{eq:r1}
R_1(\vx,t) = \scal{\vec{1}}{\left(\Phi_1(\cdot, t) \circ K_1\right)(\vx)},
\end{align}
where the non-linear function depends on~$t$, i.e., $\Phi_1\in\C^3(\R^{d_1}\times\T,\R^{d_1})$.
Consequently, also the pixel-wise non-linear functions of every feature channel~$\phi_{1j}\in\C^3(\R\times\T,\R)$, $j=1,\ldots,N_1$ depend on~$t$. 
Further, all~$\phi_{1j}$ are constructed using weighted 2D quartic spline basis functions, which are equally distributed over the input domain to ensure sufficient smoothness for gradient-based learning.
We refer to \cref{apdx:implementationDetails} for further details regarding the spline-based non-linear functions.

\subsubsection{Increasing Depth}
Since the prior~\eqref{eq:r1} essentially consists of a single layer, its expressiveness is limited to simple image features.
To increase capacity, we propose to stack multiple convolutions and parametric non-linear layers.
Then, an FoE-type prior facilitating $L$-layers reads as
\begin{align} \label{eq:RL}
R_L(\vx,t) = \scal{\vec{1}}{\left(\Phi_L(\cdot,t)\circ K_L\circ\cdots\circ\Phi_1(\cdot, t) \circ K_1\right)(\vx)}.
\end{align}
Each convolution~$K_i\colon\R^{d_{i-1}}\to\R^{d_i}$, $i=1,\ldots,L$ performs a linear combination of all input features, thereby enabling the mixing of features as typically performed in convolutional neural networks (CNNs).
In contrast to typical CNNs, we use parametric activation~$\Phi_i,\ldots,\Phi_{L-1}$ and potential~$\Phi_L$ functions that adapt to the corresponding feature channels.
At every layer~$i\in\{1,\ldots,L\}$ and for any feature channel~$j\in\{1,\ldots,N_i\}$, we employ a 2D parametric point-wise function~$\phi_{ij}\in\C^3(\R\times\T,\R)$ to non-linearly process the features.
This idea follows recent suggestions to facilitate spline-based parametric activation functions in deep CNNs~\citep{OcMe18,AzGu20}.
Further details on the parametric functions are in~\cref{apdx:implementationDetails}.

\subsection{Joint Learning using Score Matching}
Here, we elaborate on how to fit the parameters of the previously defined regularizers~$R_L\colon\X,\T,\Theta\to\R$ to the negative score of the \emph{joint} density~$\pdf{\rvy,\rt}:\X\times\T\to\R_+$ of the data.
For our previously defined degradation model, the joint density function reads as
\begin{align*}
\pdf{\rvy,\rt}(\vy,t) &= \left(\pdf{\rvx}\ast G(\vec{0},t\id)\right)(\vy) \pdf{\rt}(t)\\
&\propto \E_{\vx\sim\dist{\rvx}}\left[\exp\left(-\frac{\norm{\vy-\vx}_2^2}{2 t}\right)\right] \pdf{\rt}(t),
\end{align*}
where $\pdf{\rt}$ is the prior of the smoothing parameter.
Then, the objective function of (explicit) score matching is given by
\begin{align} \label{eq:jsm}
&J_\mathrm{SM}(\theta) = \\
&\hspace{1ex}\E_{\vy,t\sim\dist{\rvy,\rt}}\left[\tfrac12\norm{\nabla R_L(\vy,t;\theta) -\left(- \nabla\log p_{\rvy,\rt}(\vy,t)\right)}_M^2\right], \notag
\end{align}
where $\nabla$ denotes the full gradient of a function and an additional index denotes the gradient w.r.t. only this variable and 
$M\in\R^{d+1\times d+1}$ is a positive definite block-diagonal matrix, i.e.,
\[
M = \begin{pmatrix}
t\id & \vec{0} \\
\vec{0}^T & m_t
\end{pmatrix}.
\]
By applying the metric, we obtain
\begin{align}\label{eq:jsmParts}
&J_\mathrm{SM}(\theta)=\\
&\hspace{1ex}\E_{\vy,t\sim\dist{\rvy,\rt}}\Big[\tfrac{t}{2}\norm{\nabla_\vy R_L(\vy,t;\theta) -\left(- \nabla_\vy\log p_{\rvy,\rt}(\vy,t)\right)}_2^2 \notag\\
&\hspace{6ex}+\tfrac12\left(\tfrac{\partial}{\partial t}R_L(\vy,t;\theta) -\left(-\tfrac{\partial}{\partial t}\log p_{\rvy,\rt}(\vy,t)\right)\right)^2m_t\Big].\notag
\end{align}
Note that~$J_\mathrm{SM}$ decouples into a score matching objective on noisy images~$\vy$ and the smoothing parameter~$t$; the metric (in particular~$m_t>0$) enables balancing of both terms.
The scaling of the first term by~$t$ is a common variance reduction technique in denoising SM~\citep{SoEr19,HuLi21}.

To avoid the computation of the expectation over the true data in the gradient of the joint distribution, we apply denoising score matching to the noisy image term.
In addition, we replace the score matching objective w.r.t.~$t$ by its implicit pendant and get
\begin{align*}
&J_\mathrm{SM}(\theta)=\\
&\hspace{1ex}\E_{\vx,\vy,t\sim\dist{\rvx,\rvy,\rt}}\left[\tfrac{t}{2}\norm{\nabla_\vy R_L(\vy,t;\theta) -\tfrac{1}{t}(\vy-\vx)}_2^2 \right.\\
&\hspace{5ex}+ \left. \tfrac{m_t}{2}\left(\left(\tfrac{\partial}{\partial t}R_L(\vy,t;\theta)\right)^2 -2\tfrac{\partial^2}{\partial t^2}R_L(\vy,t;\theta)\right)\right] + C,
\end{align*}
where $C$ is an additive constant.
The proof can be obtained by combining the equivalence proofs of~\citet{Hy05} and~\citet{Vi11}.
To further simply the objective, we perform the change of variables~$\vy=\vx+\sqrt{t}\vn$, where~$\vn\sim\mathcal{N}(\vec{0},\id)$.
Then, we get the equivalent loss function
\begin{align} \label{eq:loss}
&J(\theta)=\\
&\hspace{1ex}\E_{\vx,\vn,t\sim\dist{\rvx,\rvn,\rt}}\tfrac{1}{2}\left[\norm{\sqrt{t}\nabla_\vy R_L(\vy,t;\theta) -\vn}_2^2\right. \notag\\ 
&\hspace{9ex}\left.+m_t\left(\left(\tfrac{\partial}{\partial t}R_L(\vy,t;\theta)\right)^2 -2\tfrac{\partial^2}{\partial t^2}R_L(\vy,t;\theta)\right)\right]. \notag
\end{align}
In contrast to typical denoising score matching-based loss functions~\citep{SoEr19,HoJa20}, this loss introduces a regularization along the smoothing direction~$t$.
In particular, the loss favors energies~$R_L$ that slowly change in this direction and are preferably convex.
Note that these properties are desirable for any gradient-based optimization scheme operating on the joint energy~$F$.

\subsubsection{Logarithmic Reparametrization}
The score-matching-based training ensures that the non-linear functions better approximate the score of the true data of the features as~$t\to\tmin$.
Thus, it is reasonable to distribute the learnable weights of~$\phi_i$ toward this regime to account for the increasing complexity.
Therefore, we facilitate the logarithmic reparametrization~$\hatt=\log(t)$, $\tminh=\log(\tmin)$, and $\tmaxh=\log(\tmax)$, in analogy to~\citet{KaMi22}.
Then, the domain~$\widehat{\T}$ is on the negative halfspace and the loss~\eqref{eq:loss} changes to
\begin{align} \label{eq:losslog}
&\widehat{J}(\theta)=\\
&\hspace{1ex}\E_{\vx,\vn,\hatt\sim\dist{\rvx,\rvn,\widehat{\rt}}}\frac12\Bigg[\norm{e^{\hatt/2}\nabla_\vy R_L(\vy,\hatt;\theta) -\vn}_2^2 \notag\\
&\hspace{8ex}+ m_t\left(\left(\tfrac{\partial}{\partial\hatt}R_L(\vy,\hatt;\theta)\right)^2 -2\tfrac{\partial^2}{\partial \hatt^2}R_L(\vy,\hatt;\theta)\right)\Bigg]. \notag
\end{align}
We highlight that the gradient and the Hessian are measured on the logarithmic domain to avoid intensive regularization toward~$\tmin$.

\section{Solving Inverse Problems using Gradually Non-convex Image Priors}
In various imaging applications, the task is to determine an underlying image~$\rvx$ given observations~$\rvz$.
The observations are related to the target through the forward problem
\[
\rvz = A \rvx + \rvzeta,
\]
where~$\rvzeta$ represents additive noise and $A$ describes the measurement process.
The simplest example is image denoising, where~$A=\id$ and the distribution of~$\rvzeta$ describe the noise type.
In the case of image inpainting, $A$ applies a binary mask to every image element, which is 1 if the associated pixel is observed and 0 otherwise, and~$\rvzeta\equiv\vec{0}$.

Frequently, the maximum a posteriori estimator is computed to approximate the target, which amounts to
\[
\max_{\vx\in\X}\left\{\pdf{\rvx\vert\rvz}(\vx\vert\vz)\propto \pdf{\rvz\vert\rvx}(\vz\vert\vx)\pdf{\rvx}(\vx)\right\}
\]
due to Bayes.
In the negative log domain, we get
\[
\min_{\vx\in\X}\left\{ -\log\pdf{\rvx}(\vx) -\log\pdf{\rvz\vert\rvx}(\vz\vert\vx) = R(\vx) + D(\vz,\vx)\right\},
\]
which is also known as the variational approach.
Here, the negative log-prior is equivalent to the regularizer~$R\colon\X\to\R$ and the negative log-likelihood equals the data fidelity term~$D\colon\Z\times\X\to\R$.
The data fidelity models the forward problem and ensures consistency to the observations, whereas, the regularizer incorporates prior knowledge of the solution.
Throughout this section, we assume that the data fidelity term has a simple proximal mapping, which is the case for many inverse problems in imaging.
In the case of image denoising with additive Gaussian noise of variance~$\sigma^2$, the data fidelity and the corresponding proximal map read as
\[
D(\vz,\vx)=\tfrac{1}{2\sigma^2}\norm{\vx-\vz}_2^2,\ \prox_{\tau D}(\vx)= \frac{\vx+\tfrac{\tau}{\sigma^2}\vz}{1+\tfrac{\tau}{\sigma^2}}    
\]
and for image inpainting we have
\[
D(\vz,\vx)=\delta(A\vx-\vz),\ \prox_{\tau D}(\vx)_i=\begin{cases}
x_i &\text{if } A_{ii} = 1\\
z_i &\text{else}
\end{cases},
\]
where~$\delta$ is the indicator function of~$\{\vec{0}\}$. 
To utilize the statistical knowledge of our learned prior~$R_L$, we next describe suitable ways to handle the additional smoothing parameter~$\hatt$.

\subsection{Joint Optimization} \label{sec:jointOptimization}

As presented in \secref{sec:gnc}, decreasing the smoothing parameter~$\hatt$ results in peakier and more non-convex energies.
Thus, there are pronounced local minima at~$\tminh$ and $\frac{\partial}{\partial\hatt}R_L$ is likely to point toward~$\tminh$.
Consequently, it is reasonable to minimize the joint energy also w.r.t. the smoothing parameter~$\hatt$.
Thus, we seek to solve the optimization problem
\begin{align} \label{eq:jointEnergy}
\min_{\vx\in\X,\ \hatt\in[\tminh,\tmaxh]} \left\{E(\vx,\hatt)\coloneqq R_L(\vx,\hatt) + D(\vz,\vx)\right\}.
\end{align}
A straightforward approach, requiring little knowledge of the objective, is adapting an alternating proximal gradient scheme~\citep{BoSa14}.
Further, even a Lipschitz backtracking~\citep{BeTe09} could be used because the energy can be easily evaluated.
However, we instead exploit the fact that~$\nabla_\vx R_L(\cdot,\hatt)$ is approximate $\exp(-\hatt)$-Lipschitz and propose the preconditioned proximal gradient algorithm listed in \Algref{alg:jointMinimization}.
Note that the projected gradient step w.r.t.~$\hatt$ is preconditioned by~$d^{-1}$ to account for the number of summands of~\eqref{eq:RL}.
\begin{algorithm}
\caption{Preconditioned proximal gradient for joint optimization}\label{alg:jointMinimization}
\textbf{Step 0:} Take $\vx^0\in\X\subset\R^d$, and sufficiently large $\hatt_0\in[\tminh,\tmaxh]$.
Choose~$\eta>0$\\
\textbf{Step i:} $(i\geq 0)$ iterate
\begin{align*}
\vx_{i+1} &= \prox_{\eta e^{\hatt_i} D(\vz,\cdot)}(\vx_i-\eta e^{\hatt_i}\nabla_\vx R_L(\vx_i,\hatt_i))\\
\hatt_{i+1} &= \proj_{[\tminh,\tmaxh]}\left(\hatt_i-\frac{\eta}{d}\nabla_\hatt R_L(\vx_i,\hatt_i)\right)
\end{align*}
\end{algorithm}

\subsection{Predefined Smoothing Schedule} \label{sec:predefinedSchedule}
The second approach is motivated by SBGMs~\citep{SoEr19,HoJa20}.
It is equivalent to the GNC flow presented in~\Algref{alg:graduatedNC}.
However, the update step is replaced by the proximal gradient step
\begin{align} \label{eq:fixedScheme}
\vx_{i+1} = \prox_{\eta e^{\hatt_i} D(\vz,\cdot)}(\vx_i-\eta e^{\hatt_i}\nabla_\vx R_L(\vx_i,\hatt_i))
\end{align}
to account for the additional data fidelity.
Further, we use a fixed linear schedule from the initial~$\hatt_0$ to $\tminh$.


\subsection{Task-specific Learning of Smoothing Schedule} \label{sec:learnedVN}
Since it is not clear how to choose the smoothing scheduler~$\{\hatt_i\}_{i=1}^I$, why not learn it from data for a specific task?
To do so, we propose to ``unroll'' the optimization scheme~\eqref{eq:fixedScheme} for $I\in\N$ steps and learn all~$\{\hatt_i\}_{i=1}^I$ and individual step sizes~$\{\eta_i\}_{i=1}^I$ such that the final output~$\vx_I$ is close to its corresponding ground truth, i.e.,
\[
\min_{\{\hatt_i\}_{i=1}^I\subset[\tminh,\tmaxh],\{\eta_i\}_{i=1}^I\subset\R_+} \E_{\widehat{\vx},\vz\sim\dist{\rvx,\rvz}} \left\{\norm{\vx_I-\widehat{\vx}}_2^2\right\}
\]
subject to $
\vx_{i+1} = \prox_{\eta_i D(\vz,\cdot)}(\vx_i-\eta_i \nabla_\vx R_L(\vx_i,\hatt_i))
$
for~$i=1,\ldots,I-1$ and $\vx_0=\vz$. 
Since the gradient w.r.t.~$\hatt$ of~$\nabla_\vx R_L$ is smooth due to the quartic spline interpolation, any gradient-based optimization algorithm can be used for learning.

Interestingly, this observation relates the successful trainable non-linear reaction-diffusion (TNRD) models of~\citet{ChPo16} and variational networks (VNs)~\citep{KoKl17,HaKl18,EfKo20} to SBGMs.
Thus, VNs -- with temporally changing parameters across the steps -- can be interpreted as a learned proximal-gradient scheme on gradually more non-convex energies.
This relation enables an unsupervised pretraining of the prior advocated in TNRDs or VNs, followed by a task-specific fine-tuning of either just the smoothing schedule or the entire model.

In contrast to the two previous approaches, here the only hyperparameter is the number of steps~$I$, which is typically constrained by the time budget in applications.
Moreover, if only the smoothing schedule and the step sizes are learned, just a few paired training samples are required for supervised learning due to the small parameter space.

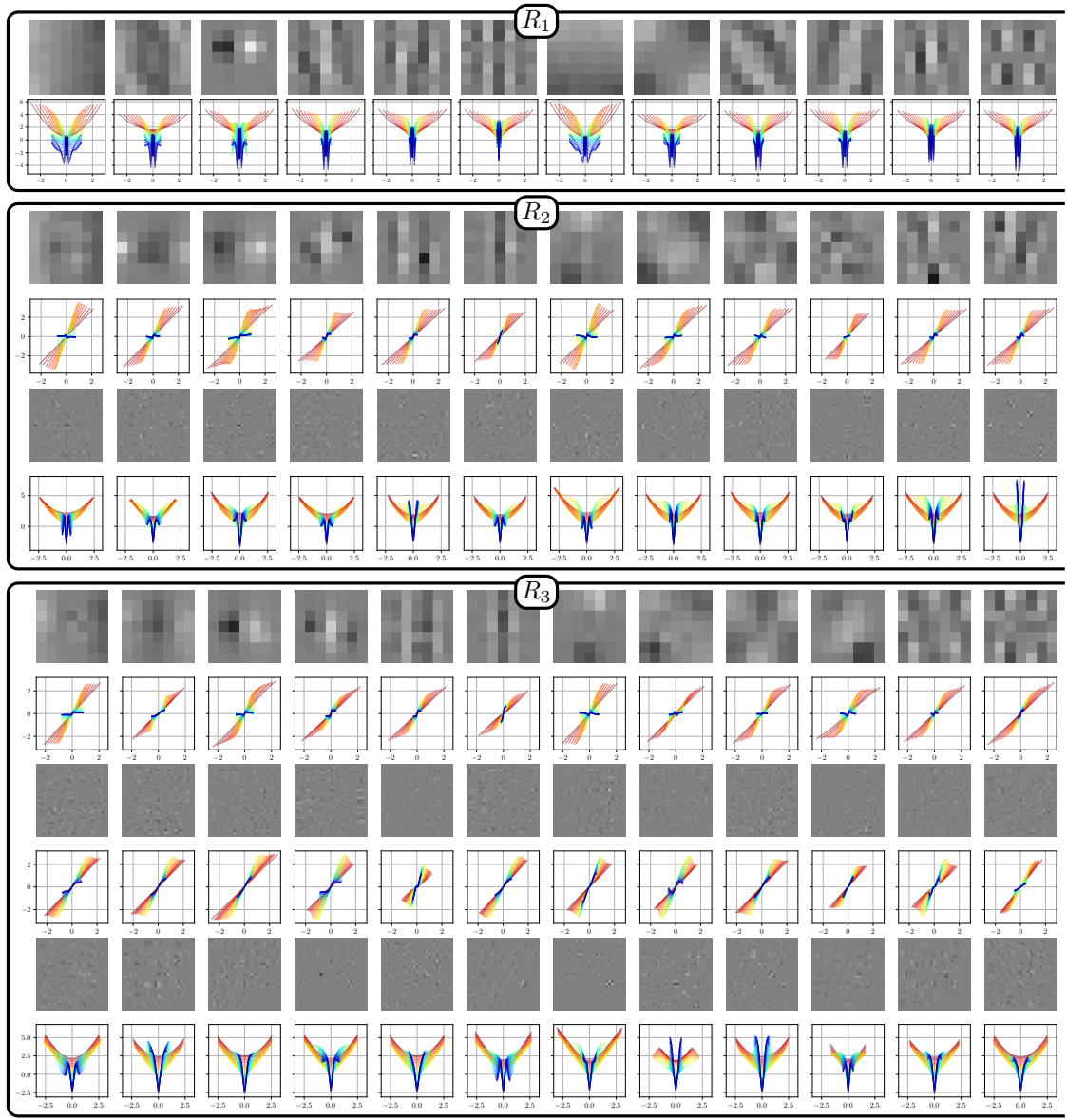
\begin{figure*}[th!]
\centering
\begin{tikzpicture}[every node/.append style={inner sep=1mm}, label/.style={draw=black,fill=white, inner sep=.5ex,rounded corners=1ex}]

\node[anchor=south] (r1) at (0,0) {\includegraphics[width=.85\linewidth]{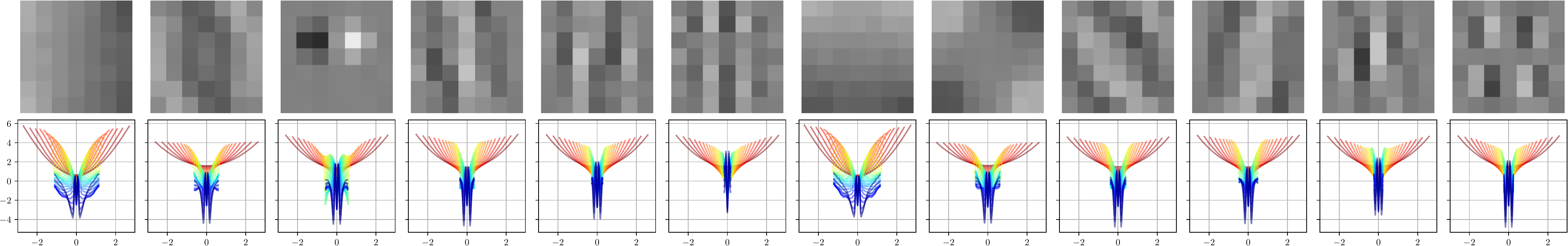}};
\draw[black, very thick, rounded corners] (r1.north east) -- node[midway,label,yshift=-1.25mm] {$R_1$} (r1.north west) -- (r1.south west) -- (r1.south east);

\node[anchor=south] (r2) at (0,-5.3) {\includegraphics[width=.85\linewidth]{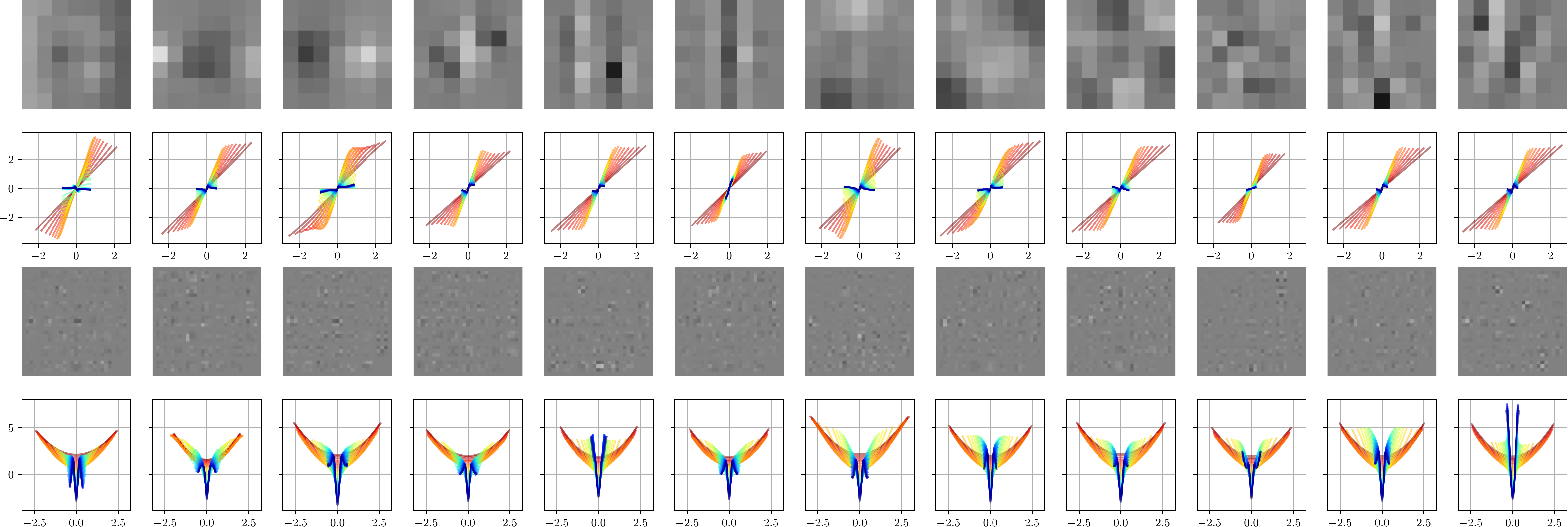}};
\draw[black, very thick, rounded corners] (r2.north east) -- node[midway,label,yshift=-1.25mm] {$R_2$} (r2.north west) -- (r2.south west) -- (r2.south east);

\node[anchor=north] (r3) at (0,-5.5) {\includegraphics[width=.85\linewidth]{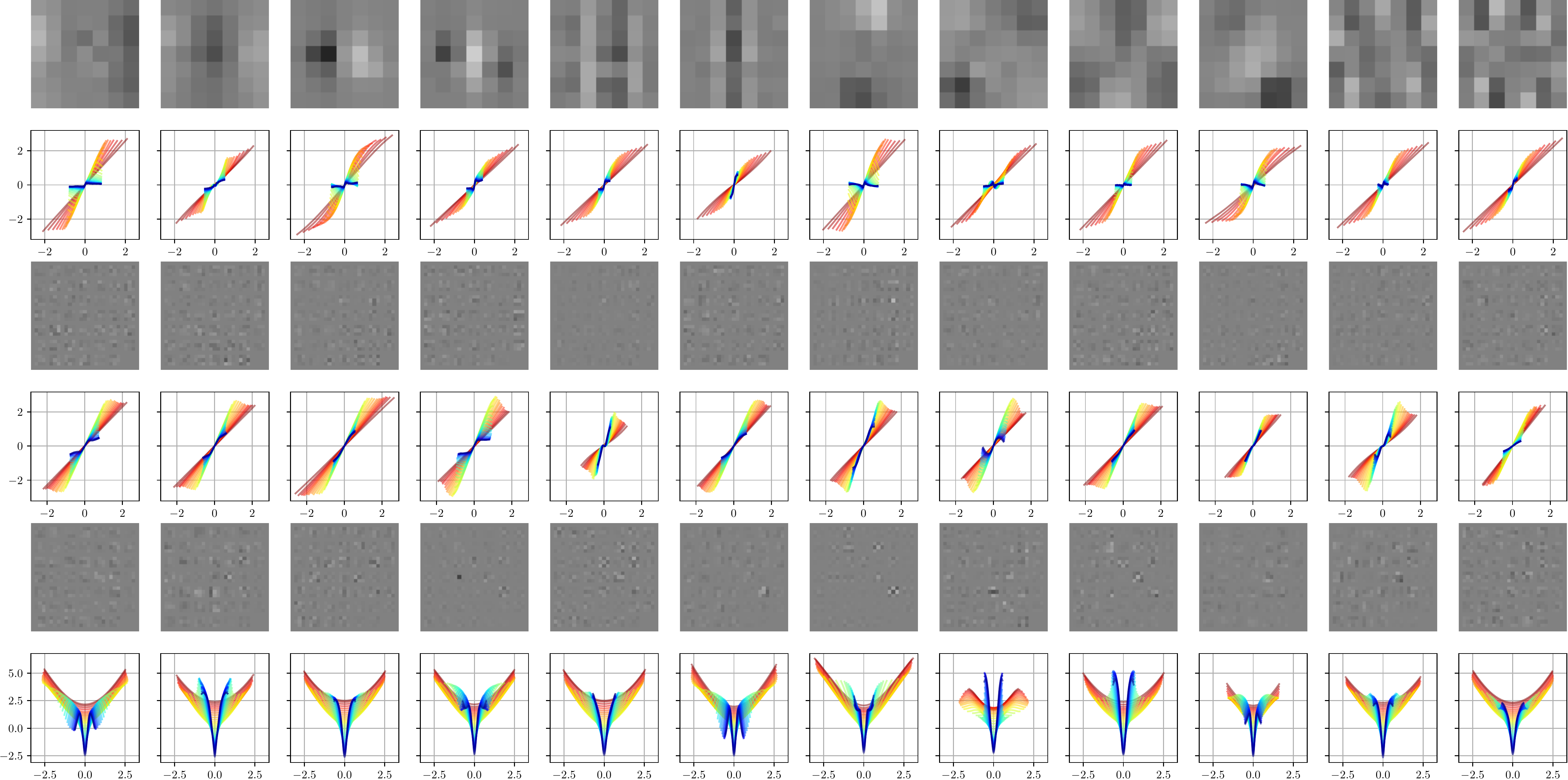}};
\draw[black, very thick, rounded corners] (r3.north east) -- node[midway,label,yshift=-1.25mm] {$R_3$} (r3.north west) -- (r3.south west) -- (r3.south east);
\end{tikzpicture}
\caption{
Visualization of the learned parameters of~$R_L$ for $L\in\{1,2,3\}$.
The columns show the first ten parameter sets of each layer.
The gray-scale images depict the convolution kernels of each input channel, while the plots below illustrate the corresponding activation functions and warmer colors indicate larger~$\hatt\in[\tminh,\tmaxh]$.
}
\label{fig:paramsRegs}
\end{figure*}

\section{Experimental Results}
In this section, we first visually analyze the learned regularizers and then compare the three previous inference techniques for image denoising and inpainting.

\subsection{Visual Analysis of Learned Priors}
Since the extended FoE priors are rather simple, a visual inspection of the parameters is still feasible.
\Figref{fig:paramsRegs} illustrates a subset of the parameters of~$R_1,R_2$, and $R_3$, where the gray-scale images show the learned convolution kernels and the plots below the corresponding 2D non-linear functions.
The blue and red colors correspond to~$\tminh$ and $\tmaxh$, respectively.
We use the effective range of the data defined by the $.01$ and $.99$-quantiles to show the effective domain of every activation function for various~$\hatt$.
This domain increases along $\hatt$ due to the variance exploding setup.
The layers of the regularizers are visualized from top to bottom.

The potential functions~$\phi_L$ (bottom of each box) nicely illustrate the GNC effect and the red functions ($\tmaxh$) become more quadratic.
Comparing the potential functions of the different priors at~$\tminh$ (blue lines), we observe that with an increasing number of layers~$L$, they become simpler in the sense that the number of local minima gets smaller and the support regions of each minimum are increased.
This probably originates from the fact that the complexity is distributed over multiple layers for~$L>2$.
Moreover, we see the GNC effect also for~$R_2$ and $R_3$ not only since the potential functions become quadratic toward~$\tmaxh$ but also because the activation functions of the hidden layers get linear.

\begin{figure}[t!]
\centering
\begin{tikzpicture}[every node/.append style={inner sep=1mm}, label/.style={draw=black,fill=white, inner sep=.5ex,rounded corners=.25ex}]
\node[anchor=center] (x) at (0,0) {\includegraphics[width=.9\linewidth]{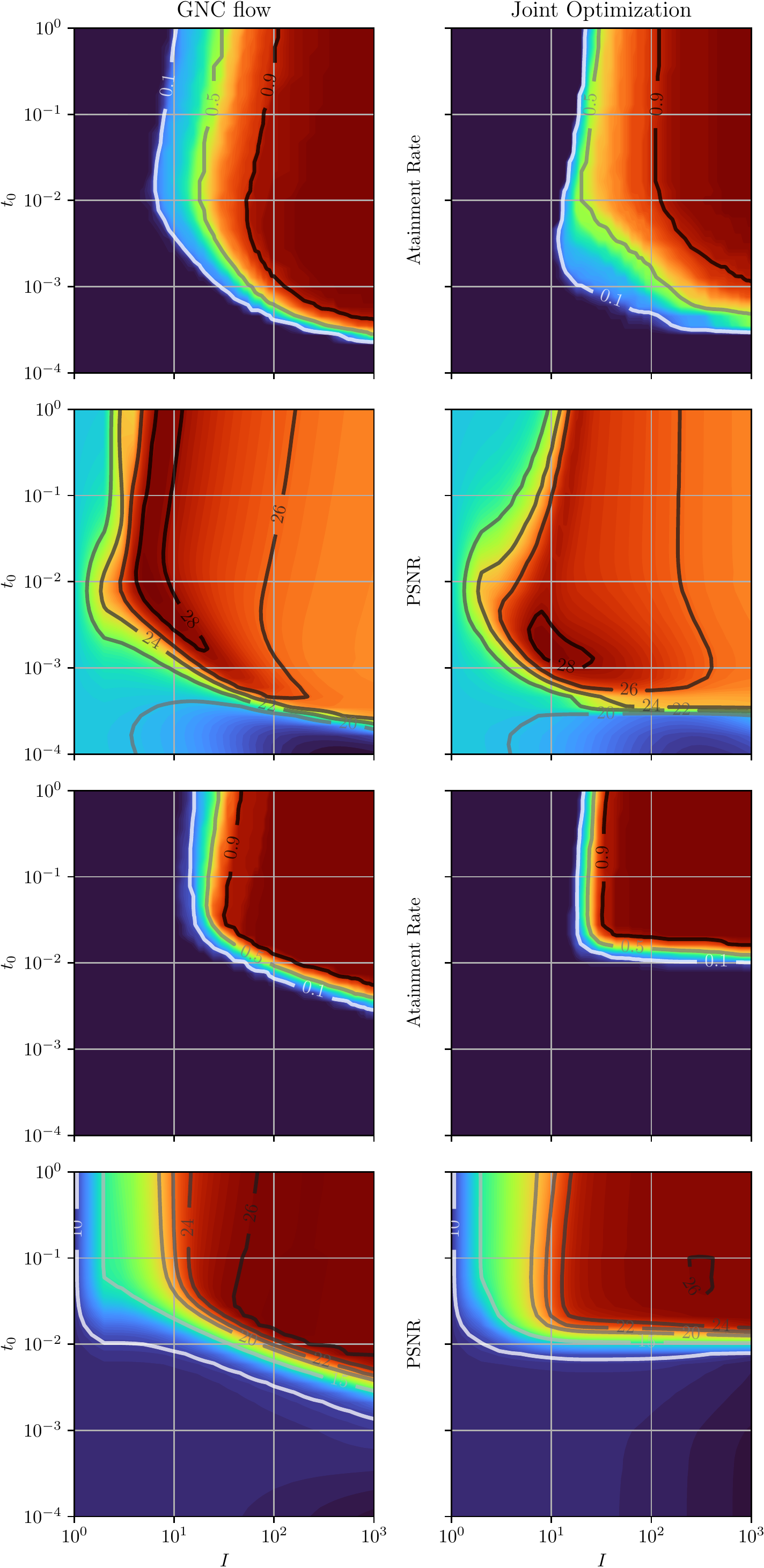}};
\draw[black, very thick] (0,0+.1) -| (x.north west) -- node[midway,label,yshift=-.75mm] {Denoising} (x.north east) |- cycle;

\draw[black, very thick] (x.south west) -- node[midway,label,yshift=.75mm] {Inpainting} (x.south east) |- (0,.08) -| cycle;
\end{tikzpicture}
\caption{
Visualization of the average global minimum attainment rate and the PSNR over the BSDS68 test set for the GNC flow (left, \eqref{eq:fixedScheme}) and the proposed joint optimization (right, \Algref{alg:jointMinimization}) for image denoising (top) and inpainting (bottom) using~$\eta=1$.
}
\label{fig:algComparison}
\end{figure}

\subsection{Ablation of Smoothing Schedule Selection}
\begin{figure}[t]
\centering
\includegraphics[width=.9\linewidth]{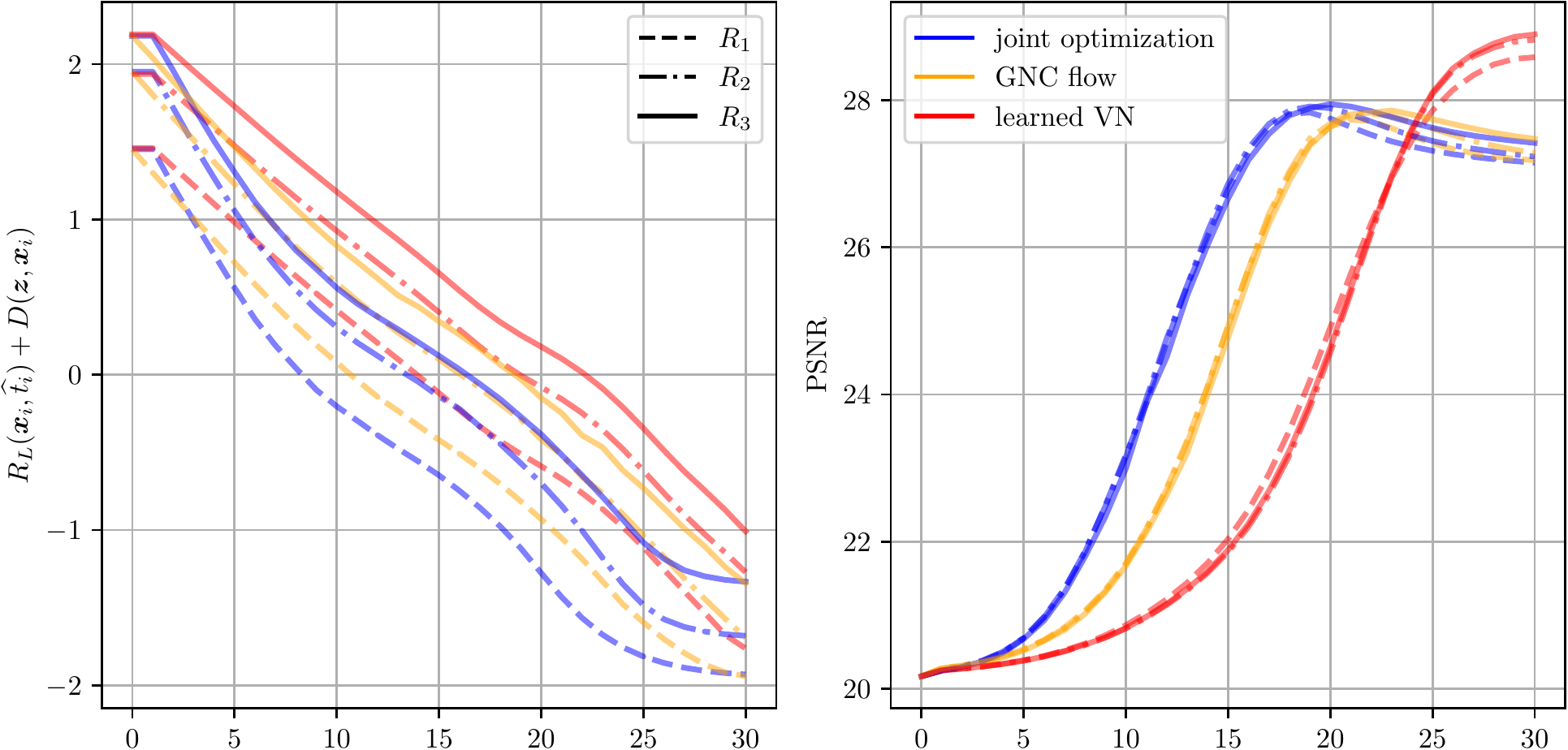}
\caption{
Comparison of the considered solution techniques for image denoising using~$I=30$ steps for the three priors~$R_L$, $L\in\{1,2,3\}$.
The left plot depicts the average energy for each step of the BSDS68 dataset, while on the right the corresponding PSNR scores are visualized.
}
\label{fig:allThreeComparison}
\end{figure}
Next, we compare the different inference techniques for solving inverse problems using the learned~$R_L$.
In \figref{fig:algComparison}, the behavior of the GNC flow~(\Secref{sec:predefinedSchedule}) and the joint optimization~(\Secref{sec:jointOptimization}) is compared for image denoising (top, $\sigma=0.1$) and inpainting (bottom, $80\%$ missing pixels).
In analogy to the 1D introductory example, the attainment rate of the global minimal energy is presented in the top row of each block to show the effectiveness of the GNC principle.
In detail, it is the average across the BSDS68~\cite{MaFo01} test set for a broad range of initial smoothing~$t_0$ and number of steps~$I$.
Comparing both algorithms, we observe only minor differences and the optimal attainment rate is obtained for sufficiently large initial smoothing~$t_0$ and number of steps~$I$.

In contrast, the average PSNR scores w.r.t.~the ground truth (second row in every box) behave differently.
In particular for image denoising, the maximum PSNR score is obtained before the global minimum energy is attained, as can be seen by comparing the plots in the first and second rows.
This originates from a modeling and optimization trade-off described by~\citet{EfKo20}, where similar models have been learned in a discriminative way.
Interestingly, this effect also arises in our generatively trained models.
The effect is not as pronounced for image inpainting since the problem is harder because the initial image~$\vx_0=\vz$ has many missing pixels set to $0$.
Thus, the initial smoothing~$\hatt_0$ and the number of steps~$I$ need to be properly selected to maximize the PSNR score.

As outlined in \Secref{sec:learnedVN}, a suitable way of determining the hyperparameters of a GNC flow is by learning them from data for a specific task.
Here, only the number of steps~$I$ is fixed a priori and the corresponding step sizes and smoothing schedule are learned to maximize performance in analogy to VNs.
\Figref{fig:allThreeComparison} illustrates the energy and PSNR curves of all three schemes for an image denoising problem ($\sigma=0.1$).
While GNC flow and the joint optimization almost attain the global minimum after $30$ steps, the learned VN does not.
However, in terms of PSNR score, the learned VN outperforms the other approaches and reaches the maximum exactly at~$30$ steps and requires only training of~$2I=60$ parameters.

\section{Conclusion}
In this work, we established connections between denoising score-based models and the GNC principle in optimization.
We found that a perturbation of the data by additive Gaussian noise leads to smoother densities respectively more convex energies with increasing variance.
We demonstrated this effect by learning a one-parameter family of FoE priors, where the smoothing parameter can be naturally incorporated into the potential functions.
We showed that these priors indeed become convex for sufficiently large smoothing and can be easily incorporated into existing approaches for solving inverse problems such as MAP estimation.
In future work, we will extend our work to more expressive priors and more challenging inverse problems.

\bibliography{references}
\bibliographystyle{icml2023}

\newpage
\appendix
\onecolumn

\section{Proofs of Theorem~\ref{thm:Fconvex} and Corollary~\ref{cor:empirical}} \label{apdx:proofs}

\begin{theorem}
Let~$\X\subset\R^d$ be a bounded set such that diameter~$\diameter(\X)<\infty$ and consider a GMM of the form
\[
p(\vx) = \sum_{i=1}^n w_i G(\vx_i,\Sigma_i)(\vx),
\]
where $(w_1\ \cdots \ w_n)^\top\in\Delta^n$. Assume~$\{\vx_i\}_{i=1}^n\subset\X$.
Then there exists a smoothing parameter~$\widetilde{t}\in(0,\infty)$ such that the smoothed energy
\[
F(\vx,t)\coloneqq -\log \big( (p \ast G(\vec{0}, t\id))(\vx)\big)
\]
is \emph{convex} w.r.t.~$\vx$ for all~$t\geq\widetilde{t}$.
\end{theorem}
\begin{proof}
The smoothed energy is defined as a convolution of Gaussians~\citep{Du19}, hence its explicit form reads as
\begin{align} \label{eq:fgmm}
F(\vx,t) = -\log \underbrace{\sum_{i=1}^n w_i G(\vx_i,\Sigma_i+t\id)(\vx)}_{\eqqcolon f(\vx,t)} = -\log f(\vx,t).
\end{align}
Since $F\in\C^{\infty}(\X\times \R_{++}, \R)$, the proof relies on showing positive definiteness of the Hessian~$\nabla_\vx^2 F(\vx,t)$ for any~$\vx\in\X$.
Let us first compute the gradient~$f$
\begin{align*}
\nabla_\vx f(\vx,t) &= 
-\sum_{i=1}^n \frac{w_i}{\vert 2\pi\widetilde{\Sigma}_i\vert^{\frac{1}{2}}} \exp\left(-\norm{\vx-\bm{\mu}}_{\widetilde{\Sigma}_i^{-1}}^2\right) \widetilde{\Sigma}_i^{-1} (\vx-\vx_i)\\
&= -\sum_{i=1}^n w_i G(\vx_i,\widetilde{\Sigma}_i) \underbrace{\widetilde{\Sigma}_i^{-1} (\vx-\vx_i)}_{\vr_i}
\end{align*}
using $\widetilde{\Sigma}_i = \Sigma_i+t\id$.
Similarly, the Hessian of $f$ is given by
\[
\nabla_\vx^2 f(\vx,t) = -\sum_{i=1}^n w_i G(\vx_i,\widetilde{\Sigma}_i)\left(\widetilde{\Sigma}_i^{-1} -  \vr_i\vr_i^\top\right).
\]
Then the gradient of the energy reads as
\[
\nabla_\vx F(\vx, t) = -\frac{1}{f(\vx,t)} \nabla_\vx f(\vx,t)
\]
and its Hessian is defined as
\[
\nabla_\vx^2 F(\vx, t) =
-\frac{1}{f(\vx,t)} \nabla_\vx^2 f(\vx, t) + \frac{1}{f(\vx,t)^2} \nabla_\vx f(\vx,t) \left(\nabla_\vx f(\vx,t)\right)^\top.
\]
By plugging in the Hessian of $f$, we get
\[
\nabla_\vx^2 F(\vx, t) = \underbrace{\frac{1}{f(\vx,t)}}_{\geq0}\Bigg\{\sum_{i=1}^n \underbrace{w_i G(\vx_i,\widetilde{\Sigma}_i)(\vx)}_{\geq0} \left( \widetilde{\Sigma}_i^{-1} - \vr_i\vr_i^\top  \right) + \underbrace{\frac{1}{f(\vx,t)} \nabla_\vx f(\vx,t) \left(\nabla_\vx f(\vx,t)\right)^\top}_{\succeq 0} \Bigg\}.
\]
For any~$t\in(0,\infty)$, the energy~$F$ is convex if $\nabla_\vx^2 F(\vx, t)\succeq0$ for all~$\vx\in\X$.
Since almost all parts of~$\nabla_\vx^2 F(\vx, t)$ are positive, we only need to ensure that
\[
\widetilde{\Sigma}_i^{-1} - \vr_i\vr_i^\top = \widetilde{\Sigma}_i^{-1} - \widetilde{\Sigma}_i^{-1} (\vx-\vx_i)(\vx-\vx_i)^\top \widetilde{\Sigma}_i^{-1} \succeq 0
\]
for all~$i=1,\ldots,n$.
By multiplying $\widetilde{\Sigma}_i$ from both sides, we get
\[
\widetilde{\Sigma}_i - (\vx-\vx_i)(\vx-\vx_i)^\top \succeq 0 \iff \Sigma_i + t\id \succeq (\vx-\vx_i)(\vx-\vx_i)^\top.
\]
Computing the minimal Eigenvalue on the left-hand-side and the maximal Eigenvalue on the right-hand-side, we obtain
\[
\lambda_\mathrm{min}(\Sigma_i) + t \geq \lambda_\mathrm{max}\left(\norm{\vx-\vx_i}^2 \frac{\vx-\vx_i}{\norm{\vx-\vx_i}}\frac{(\vx-\vx_i)^\top}{\norm{\vx-\vx_i}}\right) = \norm{\vx-\vx_i}^2.
\]
Since $\lambda_\mathrm{min}(\Sigma_i)\geq0$ for any $i=1,\ldots,n$, it is sufficient to show that
\[
t\geq \norm{\vx-\vx_i}^2,
\]
which is the case if 
\[
t\geq \max_{x\in\X}\max_{i=1,\ldots,n} \norm{\vx-\vx_i}^2
\]
holds true.
Note that we can estimate the right-hand-side from above by the domain's diameter~$\diameter(\X)=\sup_{x,y\in\X}\norm{x-y}$.
As a result, we conclude the proof by
\[
t\geq \diameter(\X)^2.
\]
\end{proof}

\begin{corollary}
Theorem~\ref{thm:Fconvex} also holds if an empirical discrete probability measure of a dataset~$\{\vx_i\}_{i=1}^n\subset\X$, i.e.
\[
p = \frac{1}{n}\sum_{i=1}^n\delta_{\vx_i},
\]
is considered.
Here, $\delta_{\vx}$ denotes the Dirac delta measure located at~$\vx$.
\end{corollary}

\begin{proof}
Since the convolution of the empirical probability measure~$p$ with a zero-mean Gaussian results in a GMM due to the translation property of the Dirac delta function, we get $F(\vx,t) = -\log f(\vx,t)$ for
\[
f(\vx,t)=\frac{1}{n}\sum_{i=1}^n G(\vx_i, t\id).
\]
Note that this results is identical to the definition of~$f$ in \eqref{eq:fgmm} if $\Sigma_i=\bm{0}$ for $i=1,\ldots,n$.
Consequently, the proof of this corollary follows the same line of arguments as in Theorem~\ref{thm:Fconvex}.
\end{proof}

\section{Implementation Details} \label{apdx:implementationDetails}
To extract and combine features, we use the following convolution operators.
The first convolution operator~$K_1$ facilitates~$N_1=48$ kernels of size~$7\times 7$, which are initialized by the 2D discrete cosine transform basis filers as in~\citet{ChPo16,KoKl17}.
All subsequent operators~$K_i$, $i=2,\ldots,L$ implement 2-dimensional convolutions using~$N_i=48$ kernels of size~$3\times 3$.
Those filters are initialized using ``Kaiming''-normal initialization.
Neither of the convolution operators facilitates bias terms since each subsequent non-linear parametric function may adapt to its input features.

Throughout the prior models, every non-linear function is implemented using weighted spline basis functions.
In addition, each non-linearity is a function in two variables -- an input feature~$x\in\R$ and the (logarithmic) smoothing parameter~$\hatt\in[\tminh,\tmaxh]$.
For the $i^\text{th}$ layer and the $j^\text{th}$ feature channel, the output of the corresponding activation function~$\phi_{ij}$ is computed as
\[
\phi_{ij}(x,\hatt) = \sum_{l=1}^{N_x}\sum_{o=1}^{N_t} w_{ij}^{lo} \varphi\left(\frac{x-\mu_x^l}{\gamma_x}\right) \varphi\left(\frac{\hatt-\mu_x^o}{\gamma_t}\right),
\]
where $N_x,N_t\in\N$ define the number of basis functions and the weights~$w_{ij}^{lo}\in\R$ are learned during optimization.
For all functions, we place the means~$\mu_x^l$ on an equidistant grid on~$[-3.5,3.5]$ and set~$\gamma_x=\frac{7}{N_x-1}$.
Likewise, the means~$\mu_t^o$ are equally distributed on the interval~$[\tminh,\tmaxh]$ and $\gamma_t=\frac{\tmaxh-\tminh}{N_t-1}$.
The kernel~$\varphi\colon\R\to\R$ is given by the quartic spline, i.e.,
\[
\varphi(x)=\tfrac{1}{24}
\begin{cases}
11 + 12(\vert x\vert+\frac12) - 6(\vert x\vert+\frac12)^2 - 12(\vert x\vert+\frac12)^3 + 6(\vert x\vert+\frac12)^4 &\text{if } 0\leq\vert x\vert< \frac12\\
1 + 4(\frac32-\vert x\vert) + 6(\frac32-\vert x\vert)^2 + 4(\frac32-\vert x\vert)^3 - 4(\frac32-\vert x\vert)^4 &\text{if } \frac12\leq \vert x\vert< \frac32\\
(\frac52-\vert x\vert)^4 &\text{if } \frac32\leq \vert x\vert<\frac52\\
0 &\text{else}
\end{cases}.
\]
These activation functions are computationally much more expensive than simple ReLU-activations.
However, the spline-based activation functions are much more expressive and several implementation tricks can be used to compute these activation functions efficiently.
The non-linear functions of the intermediate layers~$\phi_i,\ i=1,\ldots,L-1$ are initialized to the identity function, while the last non-linear function~$\phi_L$ is initialized to the quadratic function.

As a source of natural image patches, we consider the BSDS500 dataset~\cite{MaFo01} and randomly extract~$96\times 96$ image patches to define the empirical distribution~$\dist{X}$.
In the implementation, we sample the smoothing parameter from a uniform distribution, i.e., $\hatt\sim\mathcal{U}(\tminh,\tmaxh)$.
To minimize the training loss~\eqref{eq:losslog} with the scaling factor~$m_t=\frac{1}{d}$, the AdaBelief~\citep{ZhTa20} optimizer is used for $100\ 000$ iterations using a mini-batch size of~$128$ along with an initial learning rate of $10^{-3}$.
The learning rate is annealed to~$5\cdot 10^{-5}$ using a cosine scheme.
We utilize the AdaBelief optimizer since it performs preconditioning based on local curvature information.
Thus, it is well suited to learn parameters that lie in different intervals.


\end{document}